\documentclass{article}

\usepackage[main, preprint]{neurips_2026}

\usepackage[utf8]{inputenc}
\usepackage[T1]{fontenc}
\usepackage{hyperref}
\usepackage{url}
\usepackage{booktabs}
\usepackage{amsfonts}
\usepackage{nicefrac}
\usepackage{microtype}
\usepackage{xcolor}
\usepackage{amsmath, amssymb, amsthm}
\usepackage{graphicx}
\usepackage{algorithm}
\usepackage{algorithmic}
\usepackage{float}
\usepackage{bbm}
\usepackage{multirow}
\usepackage{subcaption}
\usepackage[most]{tcolorbox}
\usepackage{enumitem}

\setlist[itemize]{nosep,leftmargin=*}
\setlist[enumerate]{nosep,leftmargin=*}

\newtheorem{theorem}{Theorem}

\newtheorem{proposition}[theorem]{Proposition}

\newtheorem{assumption}{Assumption}
\newtheorem{remark}{Remark}
\newcommand{\EE}{\mathbb{E}}

\newcommand{\KL}{\mathrm{KL}}
\newcommand{\Var}{\mathrm{Var}}
\newcommand{\tr}{\mathrm{tr}}

\title{TIP: Token Importance in On-Policy Distillation}

\author{
    Yuanda Xu\thanks{Equal contribution.}\,
    \thanks{Correspondence to \texttt{yuanda@math.princeton.edu}} \quad
    Hejian Sang\footnotemark[1] \quad
    Zhengze Zhou\footnotemark[1] \quad
    Ran He\footnotemark[1] \quad
    Zhipeng Wang \quad
    Alborz Geramifard
}
\begin{document}
\maketitle
\begin{abstract}
On-policy knowledge distillation (OPD) trains a student on its own rollouts under token-level supervision from a teacher, but not all token positions matter equally, and existing views of token importance are incomplete.
We ask: \emph{which tokens carry the most useful learning signal in OPD?}
Our answer is that informative tokens come from \emph{two regions}: positions with high student entropy, and positions with low student entropy plus high teacher--student divergence---where the student is overconfident and wrong.
Empirically, student entropy is an effective but structurally incomplete proxy. Retaining 50\% of tokens with entropy-based sampling matches or exceeds all-token training while cutting end-to-end peak training memory by up to 47\%; under more aggressive retention, memory savings reach up to 58\%.
But entropy alone misses a second important region. When we isolate low-entropy, high-divergence tokens, training on fewer than 10\% of all tokens nearly matches full-token baselines, showing that overconfident tokens carry dense corrective signal despite being nearly invisible to entropy-only rules.
We organize these findings with \emph{TIP}(Token Importance in on-Policy distillation), a two-axis taxonomy over student entropy and teacher--student divergence that explains why entropy is useful yet structurally incomplete and motivates type-aware selection rules that combine uncertainty and disagreement.
We validate this picture across three teacher--student pairs spanning Qwen3, Llama, and Qwen2.5 on MATH-500 and AIME 2024/2025, and on the DeepPlanning benchmark for long-horizon agentic planning, where Q3-only training with 20\% of tokens surpasses full-token OPD.
Our experiments are implemented by extending the open-source OPD repository \url{https://github.com/HJSang/OPSD_OnPolicyDistillation}, which provides the practical training base for reproducing this work and supports memory-efficient distillation of larger models under limited GPU budgets.
\end{abstract}

\section{Introduction}
\label{sec:intro}

Knowledge distillation transfers capability from a large teacher to a smaller student by training on the teacher's output distributions~\citep{hinton2015distilling}, and is a primary driver of the rapid growth in small-model capacity~\citep{xiao2024densing}.
In on-policy distillation (OPD), the student generates its own responses and learns from the teacher's corrections at each token~\citep{agarwal2024gkd,gu2024minillm}.
Since the student generates the context, token importance is a property of the \emph{student--teacher state} at each position.
This raises a direct question: \emph{which tokens carry the most useful learning signal?}

Our central claim is simple.
In OPD, informative tokens come from two regions of the token state space:
(1) positions with \emph{high student entropy}, where the student is uncertain and still forming its prediction;
and (2) positions with \emph{low student entropy but high teacher--student divergence}, where the student is confident but misaligned with the teacher.
The first region is easy to detect with entropy alone: retaining 50\% of tokens by entropy-based sampling already matches or improves all-token training while substantially reducing end-to-end training memory.
The second is easy to miss: under more aggressive retention, entropy-only selection discards \emph{overconfident} tokens---positions where the student is sharply peaked on a continuation that the teacher strongly disfavors---because their low entropy makes them indistinguishable from correctly solved tokens.

We organize this picture with \textbf{TIP}, a two-axis taxonomy crossing student entropy and teacher--student divergence into four quadrants (Section~\ref{sec:taxonomy}).
Conceptually, entropy is an effective but structurally incomplete proxy: it must conflate ``confident and correct'' with ``confident and wrong,'' and a parameter-free Soft-OR score fixes this blind spot (Section~\ref{sec:theory}).
Experimentally, the combined score consistently improves over entropy-only selection on mathematical reasoning and remains competitive on agentic planning, where Q3-only selection is strongest (Section~\ref{sec:experiments}).

\paragraph{Contributions.}
\begin{enumerate}
    \item We propose TIP, a two-axis taxonomy that organizes token importance by student entropy and teacher--student divergence, requiring no verification labels and no extra computation beyond the standard OPD loss.
    \item We show that entropy is an effective but structurally incomplete proxy and prove that any entropy-only score is structurally blind to overconfident tokens, while a parameter-free Soft-OR score fixes this blind spot (Propositions~\ref{prop:optimal_weight}--\ref{prop:proxy_bias}, Remark~\ref{rem:fix}).
    \item We validate the taxonomy across several datasets and model families, and show that Soft-OR consistently outperforms entropy-only selection on mathematical reasoning while remaining competitive on long-horizon agentic planning in DeepPlanning~\citep{zhang2026deepplanning}, where Q3-only selection is strongest.
\end{enumerate}

\begin{figure}[t]
\centering
\includegraphics[width=\textwidth]{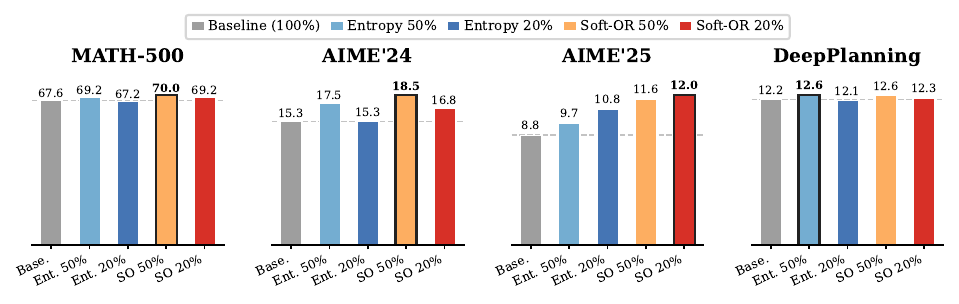}
\caption{\textbf{Cross-task summary: average accuracy by selection method.}
Each panel shows one benchmark; bar height is the mean accuracy (mean@16) averaged across three teacher--student pairs for mathematical reasoning (Qwen3-8B$\to$4B, Llama-70B$\to$8B, Qwen2.5-14B$\to$1.5B) and across two teacher sizes (14B, 32B) for DeepPlanning.
Methods: \emph{Base.}\ = all-token OPD (100\%); \emph{Ent.\ 50\%/20\%}\ = entropy-based token selection at the stated retention ratio; \emph{SO 50\%/20\%}\ = Soft-OR (Eq.~\ref{eq:score_theory}) top-$k$ selection.
The dashed line marks the all-token baseline.
Soft-OR consistently improves over entropy-only selection on the mathematical reasoning benchmarks and remains competitive on DeepPlanning, confirming that augmenting entropy with divergence recovers the Q3 blind spot without sacrificing Q1/Q2 coverage.}
\label{fig:main_results}
\end{figure}

\section{Related Work}
\label{sec:related}

\paragraph{Curriculum learning and importance sampling.}
The idea that not all training examples contribute equally dates to curriculum learning~\citep{bengio2009curriculum} and self-paced learning~\citep{kumar2010selfpaced}, which order or weight samples by difficulty.
Importance sampling extends this to gradient estimation: \citet{katharopoulos2018not} select mini-batch elements by gradient norm, and \citet{ren2018learning} learn per-example weights via meta-gradients.
These methods operate at the \emph{example} level.
Our work pushes the granularity to individual tokens within a sequence, where the relevant axes are student uncertainty and teacher--student disagreement rather than a scalar difficulty score.

\paragraph{Off-policy vs.\ on-policy distillation.}
Classical sequence-level KD~\citep{kim2016sequencekd} trains the student on teacher-generated sequences (off-policy).
On-policy distillation~\citep{agarwal2024gkd,gu2024minillm} instead lets the student generate its own rollouts and applies teacher supervision token-by-token, avoiding the train--test distribution mismatch inherent in off-policy data.
\citet{sang2026opsdc} further show that on-policy reverse KL self-distillation can compress lengthy reasoning chains into shorter ones.
Distillation has proven effective across diverse settings---from pretraining to extreme compression to expanding reasoning capabilities beyond what RL alone achieves~\citep{gu2024miniplm,xu2024onebit,yue2025doesrl}.
Because token importance in OPD is determined by the \emph{student's own} distribution at each position, it cannot be pre-computed from teacher outputs---it must be assessed online.
This makes the choice of which tokens to train on a fundamentally different problem from off-policy sample selection.

\paragraph{Response-level selection.}
Several methods operate at the sequence level: PACED~\citep{xu2026paced} utilizes a Beta-kernel weighting function based on the student model's pass rate to adaptively prioritize samples of intermediate difficulty, filtering out both trivial and overly complex data to enhance distillation efficiency, and LION~\citep{jiang2023lion} uses quality signals.
These approaches select rollouts to train on but treat all tokens within a response uniformly.
A complementary question---which we address---is which \emph{token within} a response carry the most signal.

\paragraph{Token-level importance in distillation and RL.}
In RL, \citet{wang2025beyond8020} showed that high-entropy ``forking tokens'' drive most gradient signal, \citet{cui2025entropy} further revealed that the covariance between token log-probability and advantage drives entropy collapse during policy optimization, SPINE~\citep{wu2025spine} extends this idea to test-time RL by updating only decision-critical branch points with entropy-band regularization, and \citet{xu2026ace} identified overconfident errors as a critical failure mode.
In distillation, AdaSwitch~\citep{peng2025adaswitch} switches between teacher and student guidance based on divergence, Entropy-Aware OPD~\citep{eopd2026} adapts the loss based on teacher entropy, SelecTKD~\citep{huang2025selectkd} lets the teacher verify student-proposed tokens via a propose-and-verify procedure and masks or down-weights rejected positions, LeaF~\citep{guo2025leaf} uses gradient-guided comparison with a teacher to identify and prune confounding tokens during distillation, and AdaKD~\citep{xie2026adakd} combines a divergence-based token selector (LATF, top-$r$\% by Hellinger distance) with token-level temperature scaling (IDTS).
Beyond fine-tuning, EntroDrop~\citep{wang2025entrodrop} shows that dropping low-entropy tokens during pretraining improves generalization under multi-epoch training, providing independent evidence that high-entropy positions carry most learning signal.
EDIS~\citep{zhu2026edis} further demonstrates that the \emph{temporal dynamics} of token entropy---not just its magnitude---can diagnose correct vs.\ incorrect reasoning trajectories.

Several concurrent works also explore token-level weighting, pruning, or compression across distillation and fine-tuning~\citep{wang2020minilm,tavor2026rethinking,kim2026tsdkd,wang2025winningpruning}.
Most closely related is AdaKD~\citep{xie2026adakd}, whose LATF module performs hard top-$r$\% selection by teacher--student Hellinger distance---a divergence-only view we ablate in Appendix~\ref{app:div_only}.
Our Q3 specification is not a relabeling of ``large-divergence tokens'': it is defined by the \emph{conjunction} of low student entropy and high disagreement, and the two axes induce different selections (Table~\ref{tab:entropy_sweep} vs.\ Table~\ref{tab:div_only}).
The AdaKD ablation itself supports this: LATF alone yields only $+0.04$ average ROUGE-L on Qwen2-1.5B (their Table~3a), with AdaKD's full gain coming from an orthogonal temperature-scaling module (IDTS); our budget-matched comparison (Appendix~\ref{app:div_only}) reaches the same conclusion.
Beyond this, our work proves that any entropy-only rule is structurally blind to low-entropy, high-divergence tokens (Proposition~\ref{prop:proxy_bias}), and proposes a parameter-free Soft-OR score that explicitly recovers this region, validated on mathematical reasoning and long-horizon agentic planning.

\section{Setup}
\label{sec:prelim}

Let $T$ denote a frozen teacher and $S_\theta$ a trainable student over vocabulary $V$.
A prompt $x \sim \mathcal{D}$ is drawn, the student generates a rollout $\mathbf{y} = (y_1,\ldots,y_m) \sim S_\theta(\cdot \mid x)$, and the teacher scores each position.
The context at position $t$ is $c_t = (x, y_{<t})$.
The standard on-policy distillation loss is:
\begin{equation}
    \mathcal{L} = \frac{1}{m} \sum_{t=1}^{m} D_{\KL}\!\left(P_S(\cdot \mid c_t) \,\|\, P_T(\cdot \mid c_t)\right).
    \label{eq:opd_loss}
\end{equation}

We characterize each token position by two quantities, both already computed during training:

\paragraph{Student entropy.}
\begin{equation}
    h_t = \frac{H\bigl(P_S(\cdot \mid c_t)\bigr)}{\log |V|} \in [0,1].
    \label{eq:entropy}
\end{equation}
High $h_t$ means the student is uncertain; low $h_t$ means it is confident.

\paragraph{Teacher--student divergence.}
\begin{equation}
    \delta_t = D_{\KL}\!\left(P_S(\cdot \mid c_t) \,\|\, P_T(\cdot \mid c_t)\right).
    \label{eq:divergence}
\end{equation}
High $\delta_t$ means the teacher disagrees with the student.
This is the per-token loss itself---no extra computation.

These two quantities define the plane in which we study token importance.
The empirical question of this paper is whether useful training signal concentrates in particular regions of the $(h_t, \delta_t)$ plane.

\section{TIP Taxonomy: A Two-Axis View of Token Importance}
\label{sec:taxonomy}

We organize token importance along two axes already computed during standard OPD training: student entropy $h_t$ and teacher--student divergence $\delta_t$.
Crossing them yields four quadrants (Table~\ref{tab:type_summary}, Figure~\ref{fig:taxonomy}).
The quadrants are highly imbalanced: Q4 accounts for roughly 40--47\% of all tokens, Q1 and Q2 together make up 40--52\%, and Q3 constitutes only 3--15\% across model families and datasets in the experimental setup, yet carries disproportionate corrective signal (Section~\ref{sec:q3_signal}; Appendix~\ref{app:qualitative} gives representative token-level examples, especially for Q1 and Q3).

\begin{table}[t]
\centering
\caption{\textbf{Token taxonomy.} Classification by student entropy $h_t$ and teacher--student divergence $\delta_t$.}
\label{tab:type_summary}
\small
\begin{tabular}{l cc l}
\toprule
Quadrant & $h_t$ & $\delta_t$ & Learning role \\
\midrule
Q1: High entropy, high divergence & High & High & Correct errors or consolidate fragile knowledge \\
Q2: High entropy, low divergence & High & Low & Stabilize underconfident predictions \\
Q3: Overconfident & Low & High & Break systematic confident biases \\
Q4: Solved & Low & Low & Negligible signal \\
\bottomrule
\end{tabular}
\end{table}

\begin{figure}[t]
\centering
\includegraphics[width=0.75\textwidth]{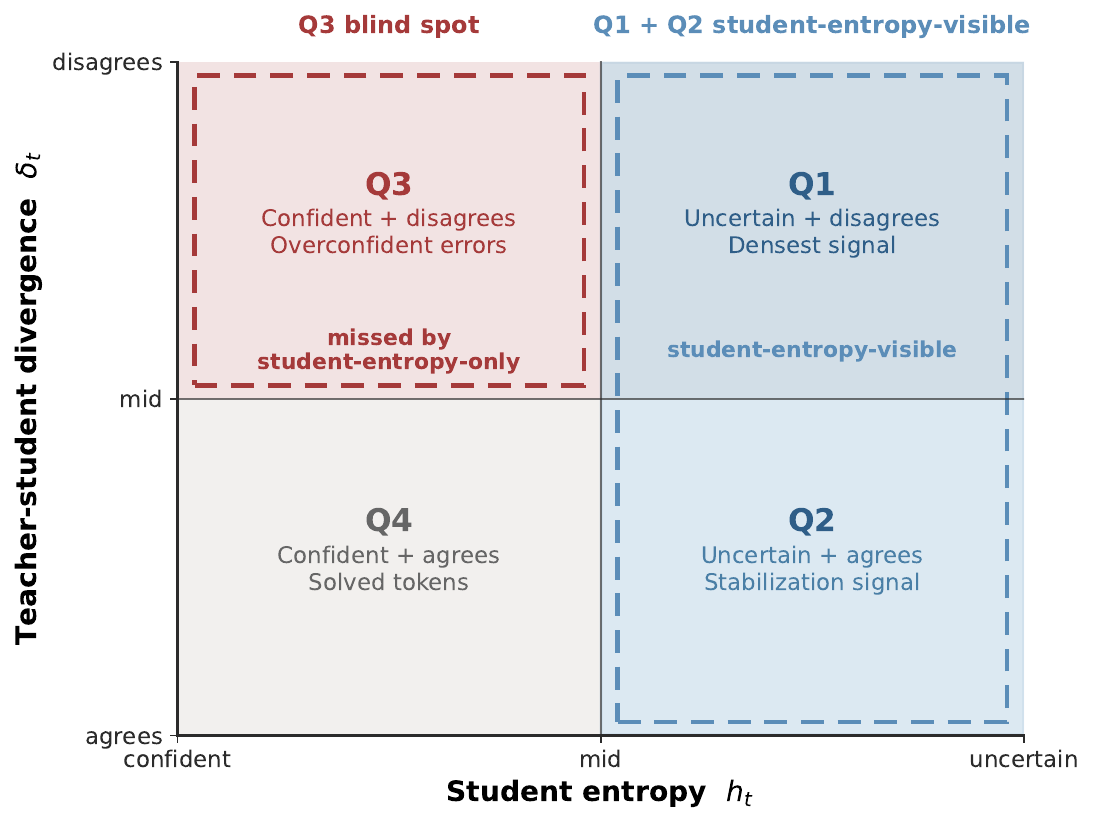}
\caption{\textbf{TIP taxonomy as a two-axis map.} Entropy determines whether the student is uncertain or confident; divergence determines whether the teacher agrees or disagrees. Q1 and Q2 are visible to entropy-based methods, while Q3 is the low-entropy blind spot that requires divergence to detect.}
\label{fig:taxonomy}
\end{figure}

\section{Theoretical Analysis}
\label{sec:theory}

The taxonomy suggests three predictions: high-entropy tokens should carry substantial learning signal relative to solved tokens (Q1/Q2 $\gg$ Q4); entropy-only selection should miss a specific class of tokens (Q3); and adding divergence should recover them.
We formalize these below and test each one experimentally in Section~\ref{sec:experiments}.
Specifically, we prove: (1) an oracle token weight suppresses solved tokens while allowing positive weight on both high-entropy and overconfident corrective tokens (Proposition~\ref{prop:optimal_weight}); (2) entropy-only scores are structurally blind to Q3 (Proposition~\ref{prop:proxy_bias}); and (3) augmenting entropy with divergence restores coverage of all informative quadrants (Remark~\ref{rem:fix}).

\subsection{Oracle Token Weight}

We want to identify which tokens most accelerate training.
We formalize this as: what per-token weights $\{w_t\}$ minimize the expected loss after one gradient step?

Let $g_t = \nabla_\theta \ell_t$ be the per-token gradient, $\bar\mu_t = \EE[g_t]$, and define
$\bar\phi_t = \langle \nabla L, \bar\mu_t \rangle$ and $\bar M_t = \EE[\|g_t\|^2]$.
Under $\beta$-smoothness and a token-separable approximation that neglects cross-token covariance terms (Appendix~\ref{app:descent_bound}), a weighted step $\hat g = \sum_t w_t g_t$ satisfies the surrogate bound:
\begin{equation}
    \EE[L(\theta - \eta\hat{g})] - L(\theta)
    \;\lesssim\;
    \sum_{t=1}^{m} \Bigl(-\eta\, w_t\, \bar\phi_t + \frac{\eta^2\beta}{2}\, w_t^2\, \bar M_t\Bigr).
    \label{eq:descent_bound_main}
\end{equation}

\begin{proposition}[\textbf{Oracle token weight}]
\label{prop:optimal_weight}
The bound is minimized at $w_t^* = \bar\phi_t / (\eta\beta\bar M_t)$,
with per-token descent $\Delta_t^* = -\bar\phi_t^2 / (2\beta\bar M_t)$.
\end{proposition}

Indeed, the bound is separable across tokens, so each coordinate minimizes
$-\eta w_t \bar\phi_t + \frac{\eta^2\beta}{2} w_t^2 \bar M_t$ independently. Differentiating gives
$-\eta \bar\phi_t + \eta^2\beta w_t \bar M_t = 0$, hence
$w_t^* = \bar\phi_t / (\eta\beta\bar M_t)$. Substituting back yields
$\Delta_t^* = -\bar\phi_t^2 / (2\beta\bar M_t)$.

This is an oracle quantity (it depends on the population gradient), but it gives a clear interpretation: informative tokens have gradients that align well with descent without excessive energy.
Across the four quadrants:
\begin{itemize}
    \item \textbf{Q1}: Large $\bar\phi_t$ (strong correction under uncertainty) $\Rightarrow$ high oracle weight when the gradient is not too noisy.
    \item \textbf{Q2}: Moderate $\bar\phi_t$ (stabilizing underconfident predictions) $\Rightarrow$ non-negligible oracle weight.
    \item \textbf{Q3}: Positive, sometimes large $\bar\phi_t$ (real corrective signal despite low entropy) $\Rightarrow$ non-negligible oracle weight that entropy-only rules miss.
    \item \textbf{Q4}: Near-zero $\bar\phi_t$ $\Rightarrow$ negligible $w_t^*$.
\end{itemize}
Thus the robust conclusion is not a fixed total ordering among Q1, Q2, and Q3; it is the separation between informative regions with positive descent signal and Q4, whose signal is near zero.

\subsection{A Signal-to-Curvature View}

The oracle bound above can be read alongside a simpler local diagnostic picture.
For a single token, consider the forward-KL/cross-entropy diagnostic loss as a function of the student's logits $z_t$, with gradient $g_t = \nabla_{z_t}\ell_t$ and Hessian $H_t = \nabla_{z_t}^2\ell_t$.
A weighted step gives the second-order approximation
\begin{equation}
    \Delta \ell_t(w_t)
    \approx
    -\eta w_t \|g_t\|^2
    + \frac{\eta^2 w_t^2}{2} g_t^\top H_t g_t,
    \label{eq:local_weight_expansion}
\end{equation}
so the local optimum satisfies
\begin{equation}
    w_t^{\mathrm{local}}
    = \frac{\|g_t\|^2}{\eta\, g_t^\top H_t g_t}.
    \label{eq:local_weight}
\end{equation}
Thus token importance is a signal-to-curvature tradeoff: the numerator measures first-order corrective signal, while the denominator measures the local second-order cost of moving on that token.

This view clarifies why entropy alone is incomplete.
High-entropy tokens often have useful corrective signal, explaining why entropy is an effective but structurally incomplete proxy.
But Q3 tokens can also have a large numerator because the teacher strongly disagrees with a confident student prediction, while their low-entropy student distribution yields small softmax curvature (Appendix~\ref{app:logit_geometry}).
Q4 tokens share the same low-curvature regime, but their teacher--student disagreement is small, so the numerator is near zero.
The distinction between Q3 and Q4 is therefore a signal distinction, not an entropy distinction.

In practice, $w_t^*$ is unavailable because it depends on population-level quantities.
A natural proxy is student entropy $h_t$, but any such score is structurally blind to Q3:

\begin{proposition}[\textbf{Blind spot}]
\label{prop:proxy_bias}
Let $\hat{w}(h_t) = f(h_t)$ be any non-decreasing score with $f(0) = 0$ (e.g., $f(h) = h$ or $f(h) = \mathbbm{1}[h \geq \tau]$). Then Q3 tokens---which may have $w_t^* > 0$---receive $\hat{w}(h_t) \approx 0$. Entropy alone cannot distinguish ``confident and correct'' (Q4) from ``confident and wrong'' (Q3).
\end{proposition}

\noindent Appendix~\ref{app:qualitative} illustrates this concretely: Examples~1, 3, and~4 show Q3 tokens with $h_t < 0.4$ that an entropy-only rule would discard, while Examples~2 and~5 show the contrasting high-entropy Q1 cases that entropy-based rules do capture.

Since divergence $\delta_t$ is already computed as part of the loss, the natural fix is a score that is nonzero whenever \emph{either} axis is active.
We define the Soft-OR score with min-max normalized inputs $\hat{h}_t, \hat{\delta}_t \in [0,1]$:
\begin{equation}
    s_t = \hat{h}_t + \hat{\delta}_t - \hat{h}_t \cdot \hat{\delta}_t
        = 1 - (1-\hat{h}_t)(1-\hat{\delta}_t).
    \label{eq:score_theory}
\end{equation}
This is parameter-free: $s_t$ is nonzero whenever either entropy or divergence is nonzero, without a tuning coefficient.

\begin{remark}[\textbf{Soft-OR fixes the blind spot}]
\label{rem:fix}
For any Q3 token with $\hat{h}_t \approx 0$ and $\hat{\delta}_t > 0$, the entropy proxy gives $\hat{w}_0(h_t) \approx 0$ (Proposition~\ref{prop:proxy_bias}), but $s_t \approx \hat{\delta}_t > 0$.
Simultaneously, Q4 tokens ($\hat{h}_t \approx 0$, $\hat{\delta}_t \approx 0$) remain suppressed: $s_t \approx 0$.
Q1 tokens retain the highest scores because both $\hat{h}_t$ and $\hat{\delta}_t$ are large ($s_t \approx 1$).
The Soft-OR score therefore preserves coverage of the high-entropy regions while recovering Q3 and suppressing Q4, without requiring $\bar\phi_t$ or $\bar M_t$.
\end{remark}

\paragraph{Empirical predictions.}
Table~\ref{tab:theory_predictions} maps each theoretical result to its experimental test.

\begin{table}[h]
\centering
\caption{\textbf{Theoretical predictions and experimental tests.}}
\label{tab:theory_predictions}
\small
\begin{tabular}{l l l}
\toprule
Result & Prediction & Tested in \\
\midrule
Proposition~\ref{prop:optimal_weight} & Q1/Q2 carry the most signal; removing Q4 improves efficiency & Section~\ref{sec:entropy_matters} \\
Proposition~\ref{prop:proxy_bias} & Entropy-only selection misses Q3 tokens & Section~\ref{sec:q3_signal} \\
Remark~\ref{rem:fix} & Combined score recovers Q3 and outperforms entropy-only & Section~\ref{sec:type_aware_results} \\
\bottomrule
\end{tabular}
\end{table}

\section{Method: Type-Aware Token Selection}
\label{sec:method}

Given a retention ratio $\rho \in (0,1]$, we retain the top-$\rho$ fraction of tokens by the Soft-OR score $s_t = \hat{h}_t + \hat{\delta}_t - \hat{h}_t \cdot \hat{\delta}_t$ (Equation~\ref{eq:score_theory}):
\begin{equation}
    \mathcal{T} = \mathrm{TopK}\bigl(\{s_t\}_{t=1}^{m},\; \lfloor \rho m \rfloor\bigr).
\end{equation}
The training loss is:
\begin{equation}
    \mathcal{L}_{\mathrm{TIP}} = \frac{1}{|\mathcal{T}|} \sum_{t \in \mathcal{T}} D_{\KL}\!\left(P_S(\cdot \mid c_t) \,\|\, P_T(\cdot \mid c_t)\right).
    \label{eq:tip_loss}
\end{equation}

Setting $\hat{\delta}_t = 0$ recovers entropy-only selection; including $\hat{\delta}_t$ additionally promotes Q3 tokens.
In practice, before computing $s_t$, we clip the top 2\% of entropy values within each batch and then apply min-max normalization to obtain $\hat{h}_t$, which suppresses rare outliers and stabilizes token ranking across batches.
The score is parameter-free and both $h_t$ and $\delta_t$ are already computed during standard distillation, so the only extra cost is this clipped min-max normalization and the top-$k$ sort---$O(m \log m)$ per rollout, negligible compared with forward and backward passes.

\section{Experiments}
\label{sec:experiments}

We now validate each prediction of the taxonomy.
Table~\ref{tab:theory_predictions} maps each theoretical result to its experimental test; we proceed from the strongest signal (high-entropy tokens, Section~\ref{sec:entropy_matters}) to the blind spot (Q3, Section~\ref{sec:q3_signal}) to the combined score (Section~\ref{sec:type_aware_results}).

\subsection{Experimental Setup}
\label{sec:exp_setup}

\paragraph{Models.}
Three teacher--student pairs across three model families for mathematical reasoning, plus one pair for agentic planning:
\begin{itemize}
    \item \textbf{Qwen3 Small:} Qwen3-8B (GRPO) $\to$ Qwen3-4B~\citep{qwen2025qwen3}
    \item \textbf{Llama:} Llama-3.3-70B-Instruct $\to$ Llama-3.1-8B-Instruct~\citep{grattafiori2024llama3}
    \item \textbf{Qwen2.5:} Qwen2.5-14B-Instruct-thinking $\to$ Qwen2.5-1.5B-Instruct~\citep{qwen2025qwen25} ($ {\sim}9\times$ capacity gap, reasoning teacher)
    \item \textbf{Qwen3 Agentic:} Qwen3-\{14B, 32B\} $\to$ Qwen3-1.7B~\citep{qwen2025qwen3} (all with thinking enabled, trained on agentic planning data)
\end{itemize}

\paragraph{Data and evaluation.}
For mathematical reasoning, training prompts are from DAPO~\citep{yu2025dapo}, with evaluation on MATH-500~\citep{hendrycks2021math} (500 problems) and AIME 2024/2025 (30 each).
For agentic planning, training data is from DeepPlanning~\citep{zhang2026deepplanning}, a benchmark featuring multi-day travel and multi-product shopping tasks that require proactive information acquisition, local constrained reasoning, and global constrained optimization; the Qwen3 Agentic pair is trained for 15 epochs.
All models are trained with AdamW, a cosine schedule, and reverse KL on student-generated rollouts (lr $= 1 \times 10^{-6}$ for Qwen3 and Qwen2.5; lr $= 3 \times 10^{-7}$ for Llama).

\subsection{High-Entropy Tokens (Q1/Q2)}
\label{sec:entropy_matters}

We begin with the simplest test of the taxonomy: if Q1/Q2 tokens dominate learning signal while Q4 tokens are negligible, then selecting by student entropy should preserve most of the benefit of OPD.

Table~\ref{tab:entropy_sweep} (and Figure~\ref{fig:entropy_sweep} in the Appendix) confirm this prediction.
Across all three model pairs, retaining 50\% of tokens with entropy-based sampling matches or outperforms the all-token baseline on most benchmarks, while total peak training memory drops substantially.
For Qwen3 Small, MATH improves from 76.7 to 78.6; for Llama, from 71.0 to 74.0.
This indicates that many low-entropy tokens are effectively solved (Q4) and mainly dilute the gradient; Appendix~\ref{app:qualitative} gives token-level intuition for the contrasting high-entropy cases that do carry corrective signal (Examples~2 and~5).

At the same time, entropy alone is incomplete. As the retention ratio becomes more aggressive, performance often drops below the full-token baseline, suggesting that useful signal remains in the discarded low-entropy region as Proposition~\ref{prop:proxy_bias} indicates. We test this hypothesis directly in the next section by isolating the low-entropy, high-divergence tokens that entropy-only selection discards.

\begin{table}[t]
\centering
\caption{\textbf{Entropy sampling across model pairs.} Accuracy (\%, mean@16 $\pm$ std). Sampling selects tokens with probability $p_t \propto h_t$. Peak Mem reports end-to-end training peak GPU memory. Bold marks the best per benchmark.}
\label{tab:entropy_sweep}
\small
\begin{tabular}{l l cccc}
\toprule
Model pair & Benchmark & 100\% & 50\% & 20\% & 10\% \\
\midrule
\multirow{3}{*}{Qwen3-8B (GRPO) $\to$ 4B}
 & MATH-500 & 76.7 $\pm$ 0.7 & \textbf{78.6} $\pm$ 0.6 & 74.1 $\pm$ 0.9 & 70.8 $\pm$ 1.0 \\
 & AIME'24 & 21.9 $\pm$ 1.2 & \textbf{23.8} $\pm$ 1.3 & 22.5 $\pm$ 1.1 & 21.1 $\pm$ 1.2 \\
 & AIME'25 & 19.4 $\pm$ 1.1 & 20.7 $\pm$ 1.3 & \textbf{21.5} $\pm$ 1.2 & 19.2 $\pm$ 1.0 \\
\midrule
\multirow{3}{*}{Llama-70B $\to$ 8B}
 & MATH-500 & 71.0 $\pm$ 0.7 & \textbf{74.0} $\pm$ 0.8 & 73.6 $\pm$ 0.7 & 72.5 $\pm$ 0.8 \\
 & AIME'24 & 21.5 $\pm$ 1.1 & \textbf{25.3} $\pm$ 1.5 & 18.8 $\pm$ 1.3 & 20.6 $\pm$ 1.4 \\
 & AIME'25 & 4.9 $\pm$ 0.9 & 7.5 $\pm$ 1.1 & \textbf{10.0} $\pm$ 1.2 & 9.3 $\pm$ 1.0 \\
\midrule
\multirow{3}{*}{Qwen2.5-14B $\to$ 1.5B}
 & MATH-500 & \textbf{55.1} $\pm$ 0.9 & 54.9 $\pm$ 0.9 & 54.0 $\pm$ 0.9 & \textbf{55.1} $\pm$ 0.9 \\
 & AIME'24 & 2.4 $\pm$ 0.7 & 3.3 $\pm$ 1.4 & \textbf{4.6} $\pm$ 1.3 & 2.7 $\pm$ 1.1 \\
 & AIME'25 & \textbf{2.1} $\pm$ 0.9 & 1.0 $\pm$ 0.5 & 1.0 $\pm$ 0.6 & 1.3 $\pm$ 0.8 \\
\midrule
\multicolumn{2}{l}{Peak Mem (GB, Qwen3)} & 72.0 & 38.1 & 35.5 & 35.3 \\
\multicolumn{2}{l}{Peak Mem (GB, Llama)} & 41.6 & 32.1 & 31.8 & 31.9 \\
\multicolumn{2}{l}{Peak Mem (GB, Qwen2.5)} & 35.8 & 19.7 & 15.0 & 14.9 \\
\bottomrule
\end{tabular}
\end{table}

\subsection{Overconfident Tokens (Q3)}
\label{sec:q3_signal}

We next test the blind-spot prediction directly by constructing the opposite of entropy-based selection: a selector that prioritizes tokens with \emph{low entropy but high teacher--student divergence}---targeting the Q3 regime in the taxonomy.

\paragraph{Q3 selection procedure.}
We isolate overconfident tokens using a confidence-weighted divergence score:
\begin{enumerate}
    \item Compute per-token forward KL: $\delta_t^{\mathrm{fwd}} = D_{\KL}(P_T(\cdot \mid c_t) \| P_S(\cdot \mid c_t))$.
    \item Compute per-token student entropy $h_t$ (Equation~\ref{eq:entropy}), clip at the 98th batch percentile $h_t \leftarrow \min(h_t, h^{(0.98)})$, and min-max normalize to $[0,1]$: $\hat{h}_t = (h_t - h_{\min}) / (h_{\max} - h_{\min})$.
    \item Define confidence as $\mathrm{conf}_t = 1 - \hat{h}_t$ (low entropy $\Rightarrow$ high confidence).
    \item Compute the Q3 score: $w_t^{\mathrm{Q3}} = \delta_t^{\mathrm{fwd}} \cdot \mathrm{conf}_t$.
\end{enumerate}
Tokens with high $w_t^{\mathrm{Q3}}$ are precisely the positions where the student is highly confident while the teacher strongly disagrees.

It is important to distinguish the roles of the two KL directions in this experiment.
Forward KL is used \emph{solely to rank tokens for selection}; the \emph{training loss} on the selected subset remains the standard reverse KL of Equation~\ref{eq:opd_loss}, $D_{\KL}(P_S \| P_T)$.
The two choices serve different purposes.
\emph{Why reverse KL for the loss:}
Reverse KL is mode-seeking---it drives the student to concentrate on the teacher's dominant mode rather than spread mass across all modes---which is the standard and numerically stable objective for on-policy distillation; it is also already computed during the forward pass, so restricting the loss to a token subset introduces no new computation.
\emph{Why forward KL for ranking:}
When the student is near-deterministic on $v^*$, reverse KL is dominated by the teacher probability assigned to the student's chosen token and has limited sensitivity to teacher-preferred alternatives that the student nearly ignores.
Forward KL, $\sum_v P_T(v)\log(P_T(v)/P_S(v))$, directly penalizes missing teacher mass and therefore produces a sharper ranking signal for Q3-type overconfidence.
We do not use forward KL as the training loss because its mode-covering bias would encourage the student to spread probability over multiple teacher-supported continuations, whereas the reverse-KL OPD baseline is the objective being compared.
The conceptual definition of Q3---low student entropy plus high teacher--student disagreement---does not depend on the KL direction; forward KL is an operational detector for this diagnostic Q3-only experiment.

\begin{table}[t]
\centering
\caption{\textbf{Training on Q3 (overconfident) tokens only.} Accuracy (\%, mean@16 $\pm$ std) when training exclusively on low-entropy, high-divergence tokens. Q3-only training with $<$10\% of all tokens can nearly match the all-token baseline across model pairs. Peak Mem reports end-to-end training peak GPU memory.}
\label{tab:q3_only}
\small
\begin{tabular}{l l ccc}
\toprule
Model pair & Benchmark & Baseline (100\%) & Q3 20\% & Q3 10\% \\
\midrule
\multirow{3}{*}{Qwen3-8B $\to$ 4B}
 & MATH-500 & \textbf{76.7} $\pm$ 0.7 & 75.7 $\pm$ 0.8 & 76.1 $\pm$ 0.9 \\
 & AIME'24 & \textbf{21.9} $\pm$ 1.2 & 20.2 $\pm$ 1.1 & 21.5 $\pm$ 1.3 \\
 & AIME'25 & \textbf{19.4} $\pm$ 1.1 & 19.2 $\pm$ 1.1 & 17.1 $\pm$ 1.0 \\
\midrule
\multirow{3}{*}{Llama-70B $\to$ 8B}
 & MATH-500 & 71.0 $\pm$ 0.7 & \textbf{71.8} $\pm$ 1.0 & 70.4 $\pm$ 1.1 \\
 & AIME'24 & \textbf{21.5} $\pm$ 1.1 & 20.2 $\pm$ 1.2 & 20.8 $\pm$ 1.1 \\
 & AIME'25 & \textbf{4.9} $\pm$ 0.9 & 4.5 $\pm$ 0.7 & 4.1 $\pm$ 0.7 \\
\midrule
\multirow{3}{*}{Qwen2.5-14B $\to$ 1.5B}
 & MATH-500 & \textbf{55.1} $\pm$ 0.9 & 54.6 $\pm$ 0.9 & 54.1 $\pm$ 0.9 \\
 & AIME'24 & 2.4 $\pm$ 0.7 & 2.5 $\pm$ 0.9 & \textbf{3.3} $\pm$ 1.4 \\
 & AIME'25 & 2.1 $\pm$ 0.9 & \textbf{3.3} $\pm$ 1.4 & 0.8 $\pm$ 0.4 \\
\midrule
\multicolumn{2}{l}{Peak Mem (GB, Qwen3)} & 72.0 & 36.4 & 36.2 \\
\multicolumn{2}{l}{Peak Mem (GB, Llama)} & 41.6 & 32.5 & 32.4 \\
\multicolumn{2}{l}{Peak Mem (GB, Qwen2.5)} & 35.8 & 19.7 & 19.6 \\
\bottomrule
\end{tabular}
\end{table}

Table~\ref{tab:q3_only} confirms that the Q3 region carries real corrective signal.
For Qwen3, training on only 5.7K overconfident tokens ($<$10\% of all tokens) reaches 76.1 on MATH, versus 76.7 for the full-token baseline.
For Qwen2.5, Q3-only training matches or exceeds the baseline on several benchmarks.
These results validate the taxonomy's prediction that Q3 tokens are informative despite having near-zero entropy.
Appendix~\ref{app:qualitative} provides concrete examples: a student that repeats a generic variable instead of substituting a concrete value (Ex.~1), an arithmetic computation error (Ex.~3), and a confident variable-level misstep in a derivation (Ex.~4)---all near-zero entropy, all strongly corrected by the teacher.

\paragraph{Q3 is not just ``large divergence''.}
Q3 tokens are by construction high-$\delta$, but the selector is structurally different from a pure divergence view.
Ranking by $\delta_t$ alone at a budget matched to Q3-only underperforms the all-token baseline and only catches up when given $5\times$ more tokens (Appendix~\ref{app:div_only}, Table~\ref{tab:div_only}); it is the low-entropy conjunct that makes selection performant under tight budgets.
Conversely, high student entropy is not a proxy for high divergence---the two axes induce different selections and different performance curves (Table~\ref{tab:entropy_sweep} vs.\ Table~\ref{tab:div_only}), and entropy-only rules are provably blind to Q3 (Proposition~\ref{prop:proxy_bias}).
A related concern is that high-$\delta$ tokens already yield the largest per-token gradients and thus should dominate the update; but in full-token training the aggregate update sums over all positions---mostly low-/moderate-$\delta$ ones---and the operationally relevant question is which tokens are worth \emph{keeping} under a budget, which Table~\ref{tab:div_only} answers in favor of the low-entropy + high-divergence selector.

\subsection{Type-Aware Selection (TIP)}
\label{sec:type_aware_results}

The prediction of Remark~\ref{rem:fix} is that combining entropy with divergence should outperform entropy-only selection by recovering Q3 without sacrificing Q1/Q2.
Table~\ref{tab:main_results} and Figure~\ref{fig:main_results} present the comparison, and Appendix~\ref{app:qualitative} provides concrete token-level examples of the Q1/Q3 behaviors that the combined score is designed to retain.

\begin{table}[t]
\centering
\caption{\textbf{Main results: Baseline vs.\ Entropy-only vs.\ Soft-OR.} Accuracy (\%, mean@16 $\pm$ std). Soft-OR uses $s_t = \hat{h}_t + \hat{\delta}_t - \hat{h}_t \cdot \hat{\delta}_t$ (Eq.~\ref{eq:score_theory}) with Top-K selection. Bold marks the best per benchmark.}
\label{tab:main_results}
\resizebox{\textwidth}{!}{%
\begin{tabular}{l l c cc cc}
\toprule
 & & Baseline & \multicolumn{2}{c}{Entropy-only} & \multicolumn{2}{c}{Soft-OR} \\
\cmidrule(lr){4-5} \cmidrule(lr){6-7}
Model pair & Benchmark & 100\% & 50\% & 20\% & 50\% & 20\% \\
\midrule
\multirow{3}{*}{Qwen3-8B $\to$ 4B}
 & MATH-500 & 76.7 $\pm$ 0.7 & 78.6 $\pm$ 0.6 & 74.1 $\pm$ 0.9 & \textbf{79.1} $\pm$ 0.8 & 77.6 $\pm$ 0.7 \\
 & AIME'24 & 21.9 $\pm$ 1.2 & 23.8 $\pm$ 1.3 & 22.5 $\pm$ 1.1 & \textbf{25.7} $\pm$ 1.4 & 24.5 $\pm$ 1.2 \\
 & AIME'25 & 19.4 $\pm$ 1.1 & 20.7 $\pm$ 1.3 & 21.5 $\pm$ 1.2 & 21.9 $\pm$ 1.2 & \textbf{23.2} $\pm$ 1.2 \\
\midrule
\multirow{3}{*}{Llama-70B $\to$ 8B}
 & MATH-500 & 71.0 $\pm$ 0.7 & 74.0 $\pm$ 0.8 & 73.6 $\pm$ 0.7 & \textbf{74.7} $\pm$ 1.0 & 74.2 $\pm$ 0.7 \\
 & AIME'24 & 21.5 $\pm$ 1.1 & 25.3 $\pm$ 1.5 & 18.8 $\pm$ 1.3 & \textbf{26.0} $\pm$ 1.4 & 21.0 $\pm$ 1.5 \\
 & AIME'25 & 4.9 $\pm$ 0.9 & 7.5 $\pm$ 1.1 & 10.0 $\pm$ 1.2 & \textbf{11.5} $\pm$ 1.1 & 10.9 $\pm$ 1.4 \\
\midrule
\multirow{3}{*}{Qwen2.5-14B $\to$ 1.5B}
 & MATH-500 & 55.1 $\pm$ 0.9 & 54.9 $\pm$ 0.9 & 54.0 $\pm$ 0.9 & \textbf{56.2} $\pm$ 1.2 & 55.8 $\pm$ 0.9 \\
 & AIME'24 & 2.4 $\pm$ 0.7 & 3.3 $\pm$ 1.4 & 4.6 $\pm$ 1.3 & 3.8 $\pm$ 1.2 & \textbf{5.0} $\pm$ 1.3 \\
 & AIME'25 & \textbf{2.1} $\pm$ 0.9 & 1.0 $\pm$ 0.5 & 1.0 $\pm$ 0.6 & 1.5 $\pm$ 0.7 & 1.8 $\pm$ 0.6 \\
\bottomrule
\end{tabular}}
\end{table}

\paragraph{Top vs.\ bottom tokens by Soft-OR score.}
A natural sanity check is the complementary experiment: instead of training on the top 50\% tokens by Soft-OR score, train on the \emph{bottom} 50\%.
If the taxonomy is correct, these tokens should be predominantly Q4 (solved) and carry negligible learning signal.

\begin{table}[t]
\centering
\caption{\textbf{Top 50\% vs.\ bottom 50\% by Soft-OR score.} Accuracy (\%, mean@16 $\pm$ std). ``Top'' trains on the highest-scoring half by $s_t$; ``Bot.'' trains on the lowest-scoring half. The bottom tokens carry substantially less signal.}
\label{tab:rest_tokens}
\small
\begin{tabular}{l cc cc cc}
\toprule
 & \multicolumn{2}{c}{Qwen3-8B $\to$ 4B} & \multicolumn{2}{c}{Llama-70B $\to$ 8B} & \multicolumn{2}{c}{Qwen2.5-14B $\to$ 1.5B} \\
\cmidrule(lr){2-3}\cmidrule(lr){4-5}\cmidrule(lr){6-7}
Benchmark & Top 50\% & Bot.\ 50\% & Top 50\% & Bot.\ 50\% & Top 50\% & Bot.\ 50\% \\
\midrule
MATH-500 & 79.1 $\pm$ 0.8 & 72.3 $\pm$ 0.9 & 74.7 $\pm$ 1.0 & 67.4 $\pm$ 1.1 & 56.2 $\pm$ 1.2 & 50.3 $\pm$ 1.0 \\
AIME'24  & 25.7 $\pm$ 1.4 & 15.5 $\pm$ 1.3 & 26.0 $\pm$ 1.4 & 17.2 $\pm$ 1.3 & 3.8 $\pm$ 1.2  & 1.5 $\pm$ 0.7  \\
AIME'25  & 21.9 $\pm$ 1.2 & 12.8 $\pm$ 1.2 & 11.5 $\pm$ 1.1 & 2.9 $\pm$ 0.8  & 1.5 $\pm$ 0.7  & 0.8 $\pm$ 0.5  \\
\bottomrule
\end{tabular}
\end{table}

\paragraph{Teacher entropy is uninformative.}
Teacher entropy is nearly constant across positions and therefore provides little discriminative signal for token selection; details are deferred to Appendix~\ref{app:teacher_entropy}. The useful axes are the \emph{student's} state ($h_t$) and the student--teacher gap ($\delta_t$), not the teacher's uncertainty.

\subsection{Beyond Mathematical Reasoning: Agentic Planning}
\label{sec:agentic}

The preceding experiments focus on mathematical reasoning.
To test whether TIP generalizes beyond this domain, we apply it to the DeepPlanning benchmark~\citep{zhang2026deepplanning} (Section~\ref{sec:exp_setup}). 

\paragraph{Setup.}
Using the Qwen3 Agentic pair described in Section~\ref{sec:exp_setup}, we train on 80\% of the DeepPlanning Travel Planning tasks and evaluate on the remaining 20\%.
We report the average accuracy (Avg@16) over 16 samples.
Following DeepPlanning, we score each plan by the fraction of personalized hard constraints satisfied; these scores are lower than commonsense scores because personalized requirements (e.g., budget limits, dietary restrictions) are more demanding.

\begin{table}[t]
\centering
\caption{\textbf{Agentic planning on DeepPlanning} (Qwen3-1.7B student, thinking-enabled, Avg@16 \%).
\emph{Left}: reference and entropy-only methods. \emph{Right}: Q3-only and Soft-OR.
Q3-only 20\% \emph{surpasses} full-token OPD; Soft-OR matches or exceeds entropy-only.}
\label{tab:agentic_results}
\footnotesize
\begin{minipage}[t]{0.48\textwidth}
\centering
\begin{tabular}{l cc}
\toprule
Method & Teacher 14B & Teacher 32B \\
\midrule
OPD, all tokens (100\%) & $11.7{\pm}0.07$ & $12.8{\pm}0.07$ \\
+ Entropy-only 50\% & $12.1{\pm}0.06$ & $13.1{\pm}0.07$ \\
+ Entropy-only 20\% & $11.6{\pm}0.07$ & $12.7{\pm}0.06$ \\
\bottomrule
\end{tabular}
\end{minipage}%
\hfill%
\begin{minipage}[t]{0.48\textwidth}
\centering
\begin{tabular}{l cc}
\toprule
Method & Teacher 14B & Teacher 32B \\
\midrule
OPD + Q3-only 20\% & $\mathbf{12.6}{\pm}0.07$ & $\mathbf{13.6}{\pm}0.07$ \\
OPD + Soft-OR 50\% & $12.0{\pm}0.06$ & $13.1{\pm}0.08$ \\
OPD + Soft-OR 20\% & $12.1{\pm}0.06$ & $12.6{\pm}0.07$ \\
\bottomrule
\end{tabular}
\end{minipage}
\end{table}

Table~\ref{tab:agentic_results} confirms that TIP generalizes to a fundamentally different domain, with entropy-based selection at 50\% retention matching or exceeding full-token OPD (12.1 vs.\ 11.7 for 14B; 13.1 vs.\ 12.8 for 32B).
The most striking finding is Q3: training on only 20\% of overconfident tokens \emph{surpasses} full-token OPD for both teachers (12.6 vs.\ 11.7; 13.6 vs.\ 12.8), confirming that entropy discards exactly the tokens with the densest corrective signal.
This aligns with the structure of agentic tasks: a single wrong but confident commitment---booking a closed venue, violating a budget constraint---can invalidate an entire plan, making Q3 corrections especially concentrated.
Appendix~\ref{app:agent_detail} provides Best@16 results, confirming the same pattern.

\section{Discussion and Conclusion}
\label{sec:conclusion}

TIP establishes that token importance in OPD is governed by two axes---student entropy and teacher--student divergence---and that both are necessary.
All three theoretical predictions (Propositions~\ref{prop:optimal_weight}--\ref{prop:proxy_bias}, Remark~\ref{rem:fix}) are supported across three model families and two task domains (Tables~\ref{tab:entropy_sweep}--\ref{tab:agentic_results}), with the clearest result being the Q3 blind spot: entropy-only rules provably cannot distinguish ``confident and correct'' from ``confident and wrong,'' yet fewer than 10\% of tokens selected precisely for overconfidence nearly matches full-training performance.
This picture is sharpest on DeepPlanning, where Q3-only training with 20\% of tokens \emph{surpasses} full OPD (12.6 vs.\ 11.7 Avg@16), suggesting that the value of catching overconfident errors grows when more downstream computation depends on a single committed step. More broadly, the two-axis framing applies to other settings where on-policy token-level supervision is used, including process reward fine-tuning and speculative decoding.

\paragraph{Limitations.}
Our largest teacher is 70B and rollouts reach 16K tokens; whether the quadrant structure and Q3 concentration persist at trillion-parameter scale or with the extremely long rollouts typical of agentic tool-calling pipelines remains an open question.

\bibliographystyle{plainnat}
\bibliography{references}

\newpage
\appendix

\section{Supplementary Theory}
\label{app:proofs}

\subsection{Logit-Level Geometry of Forward KL}
\label{app:logit_geometry}

This appendix gives the local diagnostic calculation behind the signal-to-curvature view in Section~\ref{sec:theory}; it is not a replacement for the reverse-KL OPD objective used for training.
For a token position, write the student distribution as $P_S = \mathrm{softmax}(z)$ and treat the teacher distribution $P_T$ as fixed.
The forward KL is
\begin{equation}
    D_{\KL}(P_T \| P_S)
    = \sum_{k=1}^{|V|} P_T(k) \log \frac{P_T(k)}{P_S(k)}.
\end{equation}
The softmax Jacobian is
\begin{equation}
    \frac{\partial P_S(k)}{\partial z_j}
    = P_S(k)\bigl(\mathbbm{1}\{j=k\} - P_S(j)\bigr).
\end{equation}
Therefore the gradient with respect to the student logits is
\begin{equation}
    \frac{\partial}{\partial z_j} D_{\KL}(P_T \| P_S)
    = -\sum_{k=1}^{|V|} P_T(k)\bigl(\mathbbm{1}\{j=k\} - P_S(j)\bigr)
    = P_S(j) - P_T(j),
\end{equation}
or, in vector form, $\nabla_z D_{\KL}(P_T \| P_S) = P_S - P_T$.
Differentiating once more gives the softmax Hessian
\begin{equation}
    H
    = \nabla_z^2 D_{\KL}(P_T \| P_S)
    = \mathrm{diag}(P_S) - P_S P_S^\top.
\end{equation}
Its trace is
\begin{equation}
    \tr(H)
    = \sum_{i=1}^{|V|} P_S(i)(1-P_S(i))
    = 1 - \sum_{i=1}^{|V|} P_S(i)^2.
\end{equation}
Thus the curvature term is small for near-deterministic student distributions and larger for diffuse student distributions.
This trace is not Shannon entropy, but it is a curvature-side proxy for the spread of the student distribution.

For a weighted local update on token $t$, the second-order approximation is
\begin{equation}
    \Delta \ell_t(w_t)
    \approx
    -\eta w_t \|g_t\|^2
    + \frac{\eta^2 w_t^2}{2} g_t^\top H_t g_t,
\end{equation}
which is minimized at
\begin{equation}
    w_t^{\mathrm{local}}
    = \frac{\|g_t\|^2}{\eta\, g_t^\top H_t g_t}.
\end{equation}
Moreover, by the Rayleigh quotient,
\begin{equation}
    g_t^\top H_t g_t
    \leq \lambda_{\max}(H_t)\|g_t\|^2
    \leq \tr(H_t)\|g_t\|^2.
\end{equation}
This bound should not be read as an entropy-only weighting rule: small curvature is valuable only when the numerator $\|g_t\|^2$ is non-negligible.
Q3 and Q4 can both have low entropy and low curvature, but Q3 has large teacher--student disagreement while Q4 does not.

\subsection{Derivation of the Descent Bound}
\label{app:descent_bound}

\begin{assumption}[\textbf{Smoothness}]
\label{asm:smooth}
$L(\theta') \leq L(\theta) + \langle \nabla L(\theta), \theta'-\theta\rangle + \frac{\beta}{2}\|\theta'-\theta\|^2$.
\end{assumption}

\begin{assumption}[\textbf{Token-separable approximation}]
\label{asm:decorr}
For tractability, we neglect off-diagonal gradient interactions across token positions. Concretely, for $t \neq s$ we treat the centered cross-token covariance
$$
\EE[(g_t - \bar\mu_t)(g_s - \bar\mu_s)^\top]
$$
as lower-order, so that the quadratic term admits a token-separable approximation.
\end{assumption}

\begin{proof}[Derivation]
Expand $L(\theta - \eta\hat g)$ via smoothness where $\hat g = \sum_t w_t g_t$:
\begin{equation}
    L(\theta - \eta\hat g) \leq L(\theta) - \eta\langle \nabla L, \hat g\rangle + \frac{\eta^2\beta}{2}\|\hat g\|^2.
\end{equation}
Using linearity of expectation and $\hat g = \sum_t w_t g_t$,
\begin{equation}
    \EE[-\eta\langle \nabla L, \hat g\rangle]
    = -\eta \sum_t w_t \, \EE[\langle \nabla L, g_t\rangle]
    = -\eta \sum_t w_t \, \langle \nabla L, \bar\mu_t\rangle
    = -\eta \sum_t w_t \, \bar\phi_t.
\end{equation}
For the quadratic term,
\begin{equation}
    \EE[\|\hat g\|^2]
    = \EE\Bigl[\Bigl\|\sum_t w_t g_t\Bigr\|^2\Bigr]
    = \sum_t w_t^2 \EE[\|g_t\|^2] + \sum_{t \neq s} w_t w_s \, \EE[\langle g_t, g_s\rangle].
\end{equation}
Write $g_t = \bar\mu_t + (g_t - \bar\mu_t)$. Under Assumption~\ref{asm:decorr}, we drop the off-diagonal covariance contribution and treat the remaining mean interaction terms as lower-order; absorbing these into the $\lesssim$ notation gives the token-separable approximation
\begin{equation}
    \EE[\|\hat g\|^2] \lesssim \sum_t w_t^2 \bar M_t.
\end{equation}
Combining the two displays gives
\begin{equation}
    \EE[L(\theta - \eta\hat g)] - L(\theta)
    \lesssim
    \sum_t \Bigl(-\eta w_t \bar\phi_t + \frac{\eta^2\beta}{2} w_t^2 \bar M_t\Bigr).
\end{equation}
The right-hand side is separable in $t$, so minimizing each term gives
\begin{equation}
    \frac{\partial}{\partial w_t}\Bigl(-\eta w_t \bar\phi_t + \frac{\eta^2\beta}{2} w_t^2 \bar M_t\Bigr)
    = -\eta \bar\phi_t + \eta^2\beta w_t \bar M_t = 0,
\end{equation}
which yields $w_t^* = \bar\phi_t/(\eta\beta\bar M_t)$.
\end{proof}

\subsection{Entropy-Weighted Sampling: Coverage and Variance}

Sampling tokens with $p_t \propto h_t$ and using the importance-weighted estimator $\hat{g}_{\mathrm{IS}} = \frac{1}{m}\sum_{t \in S} g_t / p_t$ is unbiased whenever $p_t > 0$, with variance
\begin{equation}
    \Var(\hat{g}_{\mathrm{IS}}) = \frac{1}{m^2}\sum_{t=1}^m \frac{1-p_t}{p_t}\, \EE[\|g_t\|^2].
\end{equation}
This makes the tradeoff transparent: entropy sampling preserves nonzero Q3 coverage, but the variance cost grows as $1/p_t$ for low-entropy tokens.

\subsection{Why Adding Divergence Improves the Ranking}

Under entropy-only scoring $\hat w_0 = h_t$, Q3 and Q4 are both pushed to the low-score end.
The Soft-OR score $s_t = \hat{h}_t + \hat{\delta}_t - \hat{h}_t \cdot \hat{\delta}_t$ separates them: Q3 receives positive score from divergence ($s_t \approx \hat{\delta}_t$) while Q4 remains near zero (both axes small).
Because the product term $\hat{h}_t \cdot \hat{\delta}_t$ prevents double-counting, the high-entropy ranking is preserved while Q3 is separated from Q4 without any tuning parameter.

\section{Supplementary Experiments}
\label{app:experiments}

\begin{figure}[H]
\centering
\includegraphics[width=\textwidth]{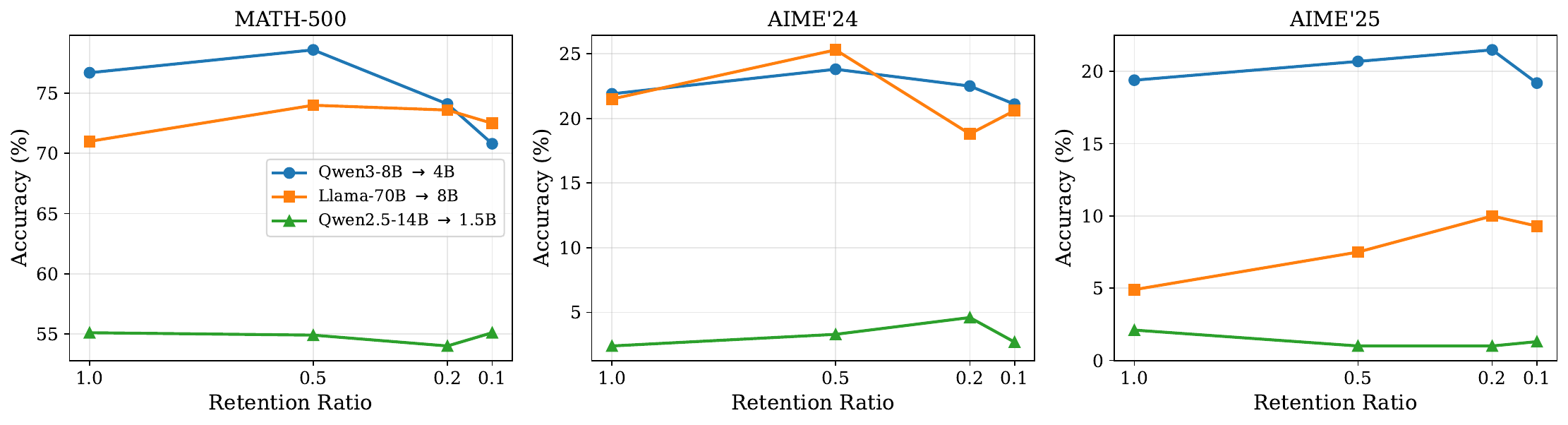}
\caption{\textbf{Entropy sampling across retention ratios.} Accuracy (mean@16) on three benchmarks as a function of retention ratio. Retaining 50\% of tokens with entropy-based sampling matches or outperforms the all-token baseline across model pairs. At very low retention, entropy-only selection begins to plateau or degrade.}
\label{fig:entropy_sweep}
\end{figure}

\subsection{Teacher Entropy Is Uninformative}
\label{app:teacher_entropy}

Teacher entropy is near-zero everywhere (mean $0.031$, std $0.055$ for Qwen3; mean $0.067$, std $0.164$ for Llama; median token probability ${\geq}0.79$), so any scheme that conditions on teacher entropy---whether for token selection, loss weighting, or sampling---receives an almost constant input and therefore adds little discriminative information. In practice, the useful axes are the \emph{student's} state ($h_t$) and the student--teacher gap ($\delta_t$), not the teacher's uncertainty.

\subsection{Isolating the Entropy Axis: Comparison with Divergence-Only Selection}
\label{app:div_only}

Section~\ref{sec:q3_signal} shows that low-entropy, high-divergence (Q3) tokens carry dense corrective signal.
A natural follow-up question is whether the \emph{low-entropy} conjunct matters at all, or whether ranking by divergence $\delta_t$ alone---ignoring student entropy---would suffice.
This is also the most direct comparison to the LATF module of AdaKD~\citep{xie2026adakd}, which selects the top-$r$\% of tokens by teacher--student Hellinger distance: our \emph{Div-only} top-$k$ selector is the same idea instantiated with reverse KL.
We deliberately give Div-only a much larger budget (top 50\%) than Q3-only (top 10\%): if the entropy axis were redundant, Div-only with $5\times$ more tokens should clearly dominate.

\paragraph{External reference point.}
The published AdaKD ablation provides an independent data point for divergence-only selection.
On Qwen2-7B$\to$1.5B (instruction-following, RKD baseline), their Table~3a reports that adding LATF alone changes ROUGE-L average from $37.03$ to $37.07$ ($+0.04$); the full +1.98 improvement of AdaKD comes from their orthogonal temperature-scaling module (IDTS), not from the token selector.
This is consistent with our Proposition~\ref{prop:proxy_bias} and motivates the present ablation in our setting.

\paragraph{Setup.}
On Qwen3-8B (GRPO) $\to$ 4B with the standard mathematical-reasoning protocol (Section~\ref{sec:exp_setup}), we compare:
\begin{itemize}
    \item \textbf{Baseline}: all tokens (100\%).
    \item \textbf{Div-only 50\%} (LATF-style hard top-$k$): top 50\% of tokens by $\delta_t$ alone---no entropy term.
    \item \textbf{Div-only 10\%}: top 10\% of tokens by $\delta_t$ alone---budget-matched to Q3-only.
    \item \textbf{Q3-only 10\%}: top 10\% of tokens by $\delta_t^{\mathrm{fwd}} \cdot (1 - \hat{h}_t)$ (Section~\ref{sec:q3_signal})---explicitly requires both low entropy and high divergence.
\end{itemize}
This design isolates two questions: (i) at $5\times$ fewer tokens, does adding the low-entropy filter beat divergence alone (Div-only 50\% vs Q3-only 10\%)? and (ii) at \emph{equal} budget, does the low-entropy filter beat pure divergence ranking (Div-only 10\% vs Q3-only 10\%)?

\begin{table}[H]
\centering
\caption{\textbf{Isolating the entropy axis: Div-only vs Q3-only} on Qwen3-8B (GRPO) $\to$ 4B (mean@16 \%, $\pm$ std). Div-only is a hard top-$k$ instantiation of LATF-style divergence selection~\citep{xie2026adakd}. We compare Q3-only at 10\% against Div-only at both $5\times$ the budget (50\%) and matched budget (10\%). Bold marks the best per benchmark.}
\label{tab:div_only}
\small
\begin{tabular}{l l cccc}
\toprule
 & & Baseline & Div-only & Div-only & Q3-only \\
 & Benchmark & 100\% & 50\% & 10\% & 10\% \\
\midrule
\multirow{3}{*}{\emph{Accuracy}}
 & MATH-500 & \textbf{76.7} $\pm$ 0.7 & 76.3 $\pm$ 0.8 & 74.3 $\pm$ 0.8 & 76.1 $\pm$ 0.9 \\
 & AIME'24  & \textbf{21.9} $\pm$ 1.2 &  21.5 $\pm$ 1.1 &  19.2 $\pm$ 1.2 & 21.5 $\pm$ 1.3 \\
 & AIME'25  & 19.4 $\pm$ 1.1 & \textbf{21.1} $\pm$ 1.2 & 15.3 $\pm$ 1.1 & 17.1 $\pm$ 1.0 \\
\bottomrule
\end{tabular}
\end{table}

\paragraph{Results.}
Table~\ref{tab:div_only} resolves both questions and sharpens what the entropy axis contributes.

\emph{(i) Equal budget (Div-only 10\% vs Q3-only 10\%).}
At the same 10\% retention, Q3-only beats Div-only on every benchmark: MATH-500 $76.1$ vs $74.3$ ($+1.8$), AIME'24 $21.5$ vs $19.2$ ($+2.3$), AIME'25 $17.1$ vs $15.3$ ($+1.8$).
Div-only at 10\% drops well below the all-token baseline ($-2.4, -2.7, -4.1$), while Q3-only stays within $0.6$ of the baseline on MATH/AIME'24.
With the divergence ranking and budget held fixed, simply requiring tokens to also be \emph{low-entropy}---i.e., re-weighting by $(1-\hat h_t)$---is what closes the gap to the baseline.

\emph{(ii) $5\times$ budget gap (Div-only 50\% vs Q3-only 10\%).}
Once Div-only is given $5\times$ the tokens (50\%), it largely catches up: MATH-500 $76.3$ vs Q3-only's $76.1$ ($+0.2$), AIME'24 a tie at $21.5$, AIME'25 $21.1$ vs $17.1$ ($+4.0$ in Div-only's favor).
This is consistent with the fact that, as $r$ grows, the high-divergence tail necessarily \emph{includes} most Q3 tokens as a side effect.
The headline is per-token efficiency: Q3-only matches Div-only on two of three benchmarks while using $5\times$ fewer tokens.

\paragraph{What this comparison does and does not say.}
We are not claiming that Q3 is numerically disjoint from a ``large-divergence'' view.
Q3 is, by construction, a high-$\delta$ region; under a sharp teacher distribution it can also overlap with high-entropy student tokens.
The point is structural rather than set-theoretic.
Prior work on token-level importance has mostly interpreted ``which tokens matter'' through the lens of student uncertainty---high entropy is taken as the proxy for an informative position---and entropy-based rules are provably blind to confident-but-wrong tokens (Proposition~\ref{prop:proxy_bias}).
What our taxonomy adds is a more fine-grained decomposition: informative tokens do not come from a single homogeneous source, and in particular there is a distinct low-entropy, high-disagreement region (Q3) that entropy-based views fail to isolate.
Table~\ref{tab:div_only} makes this concrete from two directions. Large-divergence ranking on its own is \emph{not} sufficient: Div-only at a budget matched to Q3-only is consistently weaker than the joint score, and only recovers when given a $5\times$ larger budget---i.e., when it also sweeps up much of the low-entropy tail.
Conversely, large student entropy is not the same selector as large divergence: high-$h$ tokens need not be high-$\delta$, and---as the entropy sweeps in Table~\ref{tab:entropy_sweep} show---the two selectors trace different performance curves across retention ratios.
The ``low entropy + high disagreement'' specification is therefore not a redundant restatement of either axis; it is a separate, structurally identifiable region.

\paragraph{Why ``large gradient'' is not enough.}
A natural objection is that high-divergence tokens already produce the largest pointwise gradients, so they must be the dominant source of useful signal in the full-token recipe.
That intuition does not survive scrutiny under a budget.
In standard full-token training, updates are taken over \emph{all} token positions; the vast majority typically have low or moderate divergence, with a much smaller subset of high-divergence positions.
A larger gradient on an individual high-divergence token does not by itself imply that the accumulated contribution from such tokens dominates the overall update, nor that those gradients carry the most useful signal for optimization.
More importantly, the question addressed here is not whether a token has a large gradient, but whether that gradient is \emph{informative and worth keeping under a limited budget}.
The Div-only 10\% column is exactly this stress test: when forced to commit the entire budget to the highest-divergence tail, the resulting model underperforms even the all-token baseline.
The structured comparison---separating large-divergence-only from low-entropy-and-large-divergence---turns the coarse intuition that ``large-divergence tokens matter'' into a more specific and testable claim about which structurally distinct token types are actually informative under budgeted training, and identifies Q3 as the region where this distinction matters most.

\emph{Connecting back to AdaKD.}
The Div-only 50\% column is the within-our-setup analogue of AdaKD's published LATF-only ablation~\citep{xie2026adakd}: a hard top-$r$\% divergence selector with no entropy interaction.
Their Table~3a reports $+0.04$ ROUGE-L for LATF alone on Qwen2-1.5B; we observe a similarly small effect on average ($-0.4$ on MATH-500, $-0.4$ on AIME'24, $+1.7$ on AIME'25 relative to the all-token baseline).
The takeaway is the same in both settings: divergence ranking on its own is roughly baseline-equivalent at large budgets and clearly worse at small budgets, while a low-entropy + high-divergence selector is competitive at both---consistent with Proposition~\ref{prop:proxy_bias} and Remark~\ref{rem:fix}.

\subsection{Agentic Planning: Held-Out Queries}
\label{app:agent_detail}

\begin{figure}[H]
\centering
\includegraphics[width=\textwidth]{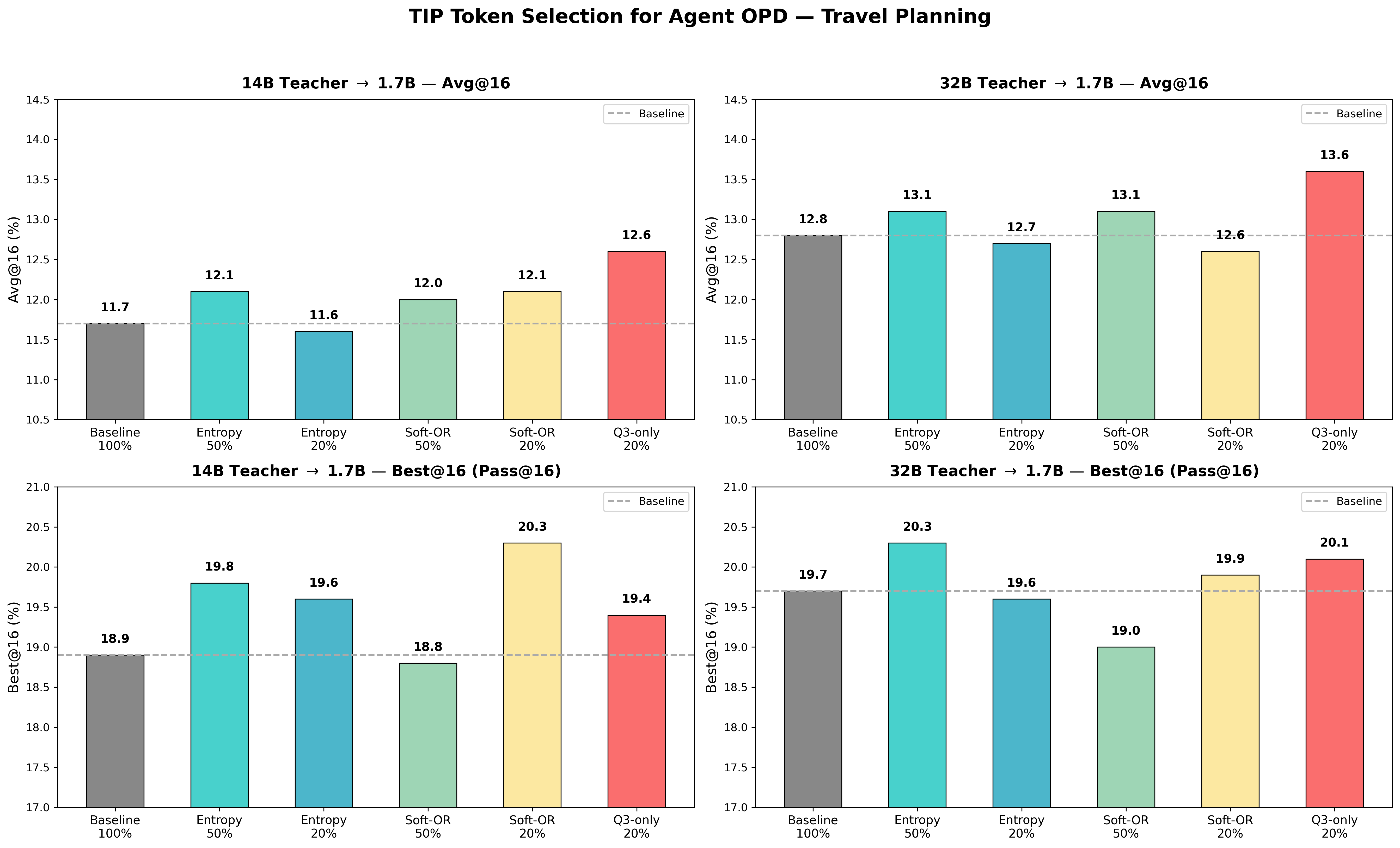}
\caption{\textbf{Token selection for agentic OPD on 20\% held-out travel-planning queries.}
\emph{Top row}: Avg@16; \emph{Bottom row}: Best@16 (Pass@16).
Within each row the left panel uses the 14B teacher and the right panel uses the 32B teacher.
Q3-only 20\% matches or exceeds the full-token baseline in every setting, consistent with Table~\ref{tab:agentic_results}.
Best@16 results show the same pattern: overconfident-token training improves the upper tail of performance, not just the mean.}
\label{fig:agent_detail}
\end{figure}

Figure~\ref{fig:agent_detail} complements Table~\ref{tab:agentic_results} with a finer-grained view.
The Avg@16 panels confirm the main-text findings: Q3-only 20\% leads for both teacher sizes (12.6 and 13.6 vs.\ baselines of 11.7 and 12.8), and entropy-only 50\% improves over full-token OPD.
The Best@16 (Pass@16) panels show the same pattern: Soft-OR 20\% achieves the highest Best@16 with the 14B teacher (20.3 vs.\ 18.9 baseline), while Q3-only 20\% leads with the 32B teacher (20.1 vs.\ 19.7).
This indicates that correcting overconfident tokens improves not just average performance but also the upper tail.

\subsection{Hyperparameters}
\label{app:hyperparameters}

Table~\ref{tab:hyperparameters} summarizes all training hyperparameters.

\begin{table}[H]
\centering
\caption{\textbf{Training hyperparameters across model pairs.}}
\label{tab:hyperparameters}
\begin{tabular}{l ccc}
\toprule
 & Qwen3 (8B$\to$4B) & Llama (70B$\to$8B) & Qwen2.5 (14B$\to$1.5B) \\
\midrule
Optimizer & AdamW & AdamW & AdamW \\
Learning rate & $1 \times 10^{-6}$ & $3 \times 10^{-7}$ & $1 \times 10^{-6}$ \\
Batch size (rollouts) & 8 & 8 & 8 \\
Rollouts per prompt & 16 & 16 & 16 \\
Max response length & 8192 & 8192 & 8192 \\
Max prompt length & 2048 & 2048 & 2048 \\
OPD chunk size & 512 & 512 & 512 \\
OPD max length & 16{,}384 & 16{,}384 & 16{,}384 \\
Tensor parallel size & 2 & 2 & 2 \\
Generation temperature & 1.0 & 1.0 & 1.0 \\
Top-$p$ (generation) & 1.0 & 1.0 & 1.0 \\
GPUs & 8$\times$H200 & 8$\times$H200 & 4$\times$H200 \\
\bottomrule
\end{tabular}
\end{table}

For entropy sampling experiments, tokens are sampled with probability $p_t \propto h_t$, with the retention ratio determining how many tokens are kept.
For Top-K experiments, the top-$\rho$ fraction of tokens by score is selected deterministically.
All evaluations use mean@16 (average accuracy over 16 independent samples per problem) with temperature 1.0.

\subsection{Qualitative Examples Across Quadrants}
\label{app:qualitative}

We present five representative tokens from training on Qwen2.5-14B $\to$ 1.5B, spanning different quadrants of the taxonomy.
Each example shows the problem, the student's response with the \colorbox{red!20}{target token highlighted}, and the student and teacher top-5 distributions.
The entropy values reported in this appendix are raw entropies in nats.

\paragraph{Example 1: Generic variable vs.\ concrete substitution (\textbf{Q3}: low entropy, high divergence).}
Student entropy: $0.02$ (extremely confident). Forward KL: $5.27$. Overconfidence score: $5.24$.

\begin{tcolorbox}[colback=gray!5, colframe=black, boxrule=0.5pt, title={Problem}]
Let $\mathcal{P}$ be a set of monic polynomials with integer coefficients of the least degree, with root $k \cdot \cos\!\left(\frac{4\pi}{7}\right)$, as $k$ spans over the positive integers. Let $P(x) \in \mathcal{P}$ be the polynomial so that $|P(1)|$ is minimized. Find the remainder when $P(2017)$ is divided by $1000$.
\end{tcolorbox}

\begin{tcolorbox}[colback=blue!3, colframe=black, boxrule=0.5pt, title={Student response (excerpt)}]
\small
\ldots The polynomial with roots of the form $\colorbox{red!20}{\textbf{k}} \cdot \cos\!\left(\frac{4\pi}{7}\right)$ can be constructed by considering the minimal polynomial of $\cos\!\left(\frac{4\pi}{7}\right)$\ldots
\end{tcolorbox}

\begin{tcolorbox}[colback=yellow!5, colframe=black, boxrule=0.5pt, title={Token distributions}]
\small
\textbf{Student top-5}: \texttt{k}~(99.8\%), \texttt{x}~(0.1\%), \texttt{e}~(0.0\%), \texttt{a}~(0.0\%), \texttt{z}~(0.0\%) \\
\textbf{Teacher top-5}: \texttt{2}~(49.9\%), \texttt{k}~(30.3\%), \texttt{x}~(8.7\%), \texttt{e}~(6.8\%), \texttt{z}~(2.8\%)
\end{tcolorbox}

\noindent\textit{Analysis.}
The student assigns 99.8\% to \texttt{k}, mechanically repeating the generic variable from the problem statement, while the teacher places 49.9\% on the concrete value \texttt{2}---indicating that the next reasoning step should substitute a specific integer rather than restate the variable.
An entropy-only rule would assign near-zero weight to this position because $h_t = 0.02$, yet it carries one of the densest corrective signals in the batch (overconfidence score $= 5.24$).

\paragraph{Example 2: Reasoning fork---restating vs.\ advancing (\textbf{Q1}: high entropy, high divergence).}
Student entropy: $1.82$ (uncertain). Forward KL: $5.31$.

\begin{tcolorbox}[colback=gray!5, colframe=black, boxrule=0.5pt, title={Problem}]
An arithmetic sequence of positive integers has $n \ge 3$ terms, initial term $a$, and common difference $d > 1$. Carl wrote down all the terms correctly except for one term, which was off by $1$. The sum of the terms he wrote was $222$. What is $a + d + n$?
\end{tcolorbox}

\begin{tcolorbox}[colback=blue!3, colframe=black, boxrule=0.5pt, title={Student response (excerpt)}]
\small
\ldots Given $a_i = a + k$ where $k = 1$ if the term $a + kd$ (for some integer $k$) was \colorbox{red!20}{\textbf{off}} and $1$ if the term $a + (k+1)d$ was off.\ldots
\end{tcolorbox}

\begin{tcolorbox}[colback=yellow!5, colframe=black, boxrule=0.5pt, title={Token distributions}]
\small
\textbf{Student top-5}: \texttt{off}~(54.4\%), \texttt{missing}~(15.6\%), \texttt{the}~(9.5\%), \texttt{not}~(4.5\%), \texttt{replaced}~(3.1\%) \\
\textbf{Teacher top-5}: \texttt{written}~(40.3\%), \texttt{increased}~(35.6\%), \texttt{off}~(7.9\%), \texttt{decreased}~(2.9\%), \texttt{incremented}~(2.0\%)
\end{tcolorbox}

\noindent\textit{Analysis.}
The student favors ``off'' (54.4\%), restating the problem, while the teacher prefers ``written'' (40.3\%) or ``increased'' (35.6\%)---words that advance the solution by characterizing the error direction.
The teacher's distribution pushes toward more precise mathematical reasoning, while the student's choice leads to circular rephrasing.
This is a classic Q1 token: the student is uncertain (entropy $= 1.82$) and the teacher strongly disagrees---entropy-based selection \emph{would} catch this token.

\paragraph{Example 3: Arithmetic computation error (\textbf{Q3}: low entropy, high divergence).}
Student entropy: $0.40$. Forward KL: $3.54$. Overconfidence score: $3.27$.

\begin{tcolorbox}[colback=gray!5, colframe=black, boxrule=0.5pt, title={Problem}]
Find the sum of the real roots of $f(x) = x^4 + 9x^3 + 18x^2 + 18x + 4$.
\end{tcolorbox}

\begin{tcolorbox}[colback=blue!3, colframe=black, boxrule=0.5pt, title={Student response (excerpt)}]
\small
\ldots $f(2) = 2^4 + 9 \cdot 2^3 + 18 \cdot 2^2 + 18 \cdot 2 + 4 = 16 + 72 + 72 + 36 + 4 = 19\colorbox{red!20}{\textbf{0}} \neq 0$\ldots
\end{tcolorbox}

\begin{tcolorbox}[colback=yellow!5, colframe=black, boxrule=0.5pt, title={Token distributions}]
\small
\textbf{Student top-5}: \texttt{0}~(91.2\%), \texttt{8}~(4.0\%), \texttt{2}~(3.5\%), \texttt{4}~(0.7\%), \texttt{6}~(0.5\%) \\
\textbf{Teacher top-5}: \texttt{6}~(60.0\%), \texttt{2}~(25.0\%), \texttt{8}~(6.3\%), \texttt{9}~(3.8\%), \texttt{0}~(3.0\%)
\end{tcolorbox}

\noindent\textit{Analysis.}
The student confidently writes $f(2) = 190$ with 91.2\% on digit \texttt{0}, committing an arithmetic error in summing $16 + 72 + 72 + 36 + 4$.
The teacher distributes probability across \texttt{6}~(60.0\%) and \texttt{2}~(25.0\%), clearly disagreeing with \texttt{0}.
Like Example~1, this is a Q3 token: the student is fairly confident ($h_t = 0.40$), yet the teacher strongly disagrees---entropy-based selection would under-weight this position.

\paragraph{Example 4: Confident on the wrong variable (\textbf{Q3}: low entropy, high divergence).}
Student entropy: $0.12$ (very confident). Forward KL: $5.58$. Overconfidence score: $5.44$.

\begin{tcolorbox}[colback=gray!5, colframe=black, boxrule=0.5pt, title={Problem}]
If the three interior angles $A, B, C$ of triangle $\triangle ABC$ have cotangents $\cot A, \cot B, \cot C$ that form an arithmetic sequence, and the maximum value of angle $B$ can be expressed as $\frac{m\pi}{n}$, find the value of $m + n$.
\end{tcolorbox}

\begin{tcolorbox}[colback=blue!3, colframe=black, boxrule=0.5pt, title={Student response (excerpt)}]
\small
\ldots Next, consider the cotangent of each angle $\cot B = \frac{\sin \colorbox{red!20}{\textbf{B}}}{\cos B}$, based on the arithmetic sequence property that $\cot B - \cot A$ and $\cot B - \cot C$ are equal\ldots
\end{tcolorbox}

\begin{tcolorbox}[colback=yellow!5, colframe=black, boxrule=0.5pt, title={Token distributions}]
\small
\textbf{Student top-5}: \texttt{B}~(98.2\%), \texttt{(B}~(1.0\%), \texttt{\textbackslash'}~(0.2\%), \texttt{\^{}}~(0.2\%), \texttt{C}~(0.1\%) \\
\textbf{Teacher top-5}: \texttt{A}~(52.4\%), \texttt{(A}~(11.7\%), \texttt{C}~(11.7\%), \texttt{B}~(6.3\%), \texttt{(}~(6.3\%)
\end{tcolorbox}

\noindent\textit{Analysis.}
The student assigns 98.2\% to \texttt{B}, while the teacher prefers \texttt{A} (52.4\%).
The student writes $\cot B = \frac{\sin B}{\cos B}$, which is mathematically incorrect ($\cot B = \frac{\cos B}{\sin B}$); the teacher's preferred continuation reflects a different and more correct reasoning path.
Another clear Q3 token: $h_t = 0.12$ would be invisible to entropy-only selection.

\smallskip\noindent\emph{Note.} The original rollout is in Chinese; the problem and response excerpts above are translated for readability. The token identifiers (\texttt{B}, \texttt{A}, etc.) are unchanged from the raw log, as mathematical symbols are language-invariant.

\paragraph{Example 5: Wrong mathematical symbol (\textbf{Q1}: moderate entropy, high divergence).}
Student entropy: $1.38$. Forward KL: $4.27$.

\begin{tcolorbox}[colback=gray!5, colframe=black, boxrule=0.5pt, title={Problem}]
(Same as Example 4)
\end{tcolorbox}

\begin{tcolorbox}[colback=blue!3, colframe=black, boxrule=0.5pt, title={Student response (excerpt)}]
\small
\ldots Assume the common difference between two adjacent terms is equal, i.e., $b - a = \backslash$\colorbox{red!20}{\textbf{Delta}} (where $\Delta$ is a constant)\ldots
\end{tcolorbox}

\begin{tcolorbox}[colback=yellow!5, colframe=black, boxrule=0.5pt, title={Token distributions}]
\small
\textbf{Student top-5}: \texttt{text}~(53.5\%), \texttt{Delta}~(25.3\%), \texttt{frac}~(13.5\%), \texttt{lambda}~(1.4\%), \texttt{pm}~(1.0\%) \\
\textbf{Teacher top-5}: \texttt{cot}~(91.7\%), \texttt{frac}~(4.6\%), \texttt{text}~(1.1\%), \texttt{delta}~(1.0\%), \texttt{Delta}~(0.8\%)
\end{tcolorbox}

\noindent\textit{Analysis.}
The teacher assigns 91.7\% to \texttt{cot}---the mathematically relevant function for this problem---while the student splits probability across irrelevant symbols (\texttt{text}, \texttt{Delta}, \texttt{frac}).
This is a Q1 token (moderate entropy, high divergence): unlike Q3 tokens, entropy-based selection \emph{would} catch it, but the teacher's near-deterministic preference for \texttt{cot} makes it an especially high-value training signal.
Note that Examples~4 and~5 come from the \emph{same student rollout} at different token positions (same Chinese-language response as Ex.~4; translated above), illustrating how a single response can contain both Q3 and Q1 tokens.

\medskip
\noindent\textbf{Summary.}
Examples 1, 3, and 4 illustrate why divergence is needed: the student's entropy is low, so any entropy-only rule would skip these tokens, yet the teacher strongly disagrees---a generic variable where a concrete value is needed (Ex.~1), an arithmetic error (Ex.~3), or a confident variable-level misstep in a derivation (Ex.~4).
Examples 2 and 5 show Q1 tokens that entropy-based selection handles well---the student is already uncertain, and the teacher provides a clear corrective signal.
The taxonomy's value is precisely that it identifies \emph{both} regions as informative while distinguishing them from Q4 (solved) tokens where both entropy and divergence are low.

\end{document}